\documentclass{article}
\usepackage{graphicx} 
\usepackage[a4paper, total={6in, 8in}]{geometry}
\usepackage{url}

\usepackage{amsfonts}
\usepackage{xcolor}
\usepackage{amsmath}
\usepackage{amsthm}

\usepackage{caption}

\usepackage{comment}

\usepackage{tikz}

\usepackage{hyperref}

\newcommand{\R}{\mathbb{R}}

\newtheorem{theorem}{Theorem}
\newtheorem{definition}{Definition}

\newtheorem{lemma}{Lemma}
\newtheorem{example}{Example}

\newcommand{\states}{\mathcal{S}}
\newcommand{\actions}{\mathcal{A}}
\newcommand{\observations}{O}

\title{General Agents Contain World Models, even under Partial Observability and Stochasticity}
\author{Santiago Cifuentes\footnote{Dovetail Research. Funded by the Advanced Research + Invention Agency (ARIA) through
project code MSAI-SE01-P005.}}
\date{\today{}}

\begin{document}

\maketitle

\begin{abstract}
     Deciding whether an agent possesses a model of its surrounding world is a fundamental step toward understanding its capabilities and limitations. In \cite{richens2025general}, it was shown that, within a particular framework, every \textit{almost optimal} and \textit{general} agent necessarily contains sufficient knowledge of its environment to allow an approximate reconstruction of it by querying the agent as a black box. This result relied on the assumptions that the agent is deterministic and that the environment is fully observable. 
     
     In this work, we remove both assumptions by extending the theorem to stochastic agents operating in partially observable environments. Fundamentally, this shows that stochastic agents cannot avoid learning their environment through the usage of randomization. We also strengthen the result by weakening the notion of generality, proving that less powerful agents already contain a model of the world in which they operate. 
\end{abstract}

\section{Introduction}

The capacity of humans to perform long horizon reasoning to complete tasks in complex, real-world environments is believed to rely heavily on rich mental representations of the world \cite{ha2018world,johnson1983mental}. Thus, it is believed that to reach human-level performance AI models will also need to have some internal representation of their surrounding environment. Explicitly model-based agents exist and have achieved impressive performance across many tasks and domains~\cite{hafner2023mastering,lecun2022path,wang2023voyager}, and even though model-free have been developed and are quite capable of generalizing across a wide range of tasks and environments~\cite{reed2022generalist,zitkovich2023rt}, there is an increasing amount of evidence that they learn implicit world models~\cite{li2022emergent}.


In this context, in \cite{richens2025general}, it was shown, under a particular framework, that from any general enough agent with long-horizon reasoning capabilities it is possible to construct a approximation of the underlying world in which the agent operates. To prove such a result, they showed that, given a capable enough agent that can be queried to perform a variety of tasks, it is possible to compute an estimator of the environment by inspecting the actions the agent takes to attain different goals. 

In their framework the world is modeled as a fully observable, finite and communicating Controlled Markov Decision Process, and the agent is seen as a deterministic goal-oriented policy. These two assumptions are exploited in the proof, and they somewhat weaken the result as real-world environments are often partially observable~\cite{kaelbling1998planning} and the use of randomization in policies is common, especially for simplifying computations of gradients and averages~\cite{sutton1999policy}.

Therefore, in this work we strengthen the result by removing these two hypotheses, extending it to partially observable environments and stochastic agents. In particular, we believe that the extension to stochastic agents is conceptually relevant, since our results show that randomized agents still have to learn a representation of their environment, although the quality of the world estimator that can be computed from them is worse than that obtained using deterministic agents. This somewhat resembles what happens in the context of randomized algorithms: it is common that some task can be solved easily and efficiently employing some stochastic strategy, while the corresponding efficient deterministic procedure requires a deeper understanding of the elements involved in the problem. This is the case for primality testing: the Miller-Rabin test~\cite{rabin1980probabilistic} can detect primality by exploiting well-known facts about modular arithmetic such as the Fermat's little theorem, while the deterministic AKS algorithm~\cite{agrawal2004primes} relies on a deeper understanding of polynomial rings and a generalization of Fermat's little theorem for them.

We also extend the original result by proving that less general agents already contain enough information to compute a model of the environment. More precisely, the original theorem shows that given an agent capable of achieving all goals of depth $n$ (where depth measures the horizon of reasoning required to achieve a goal, i.e. bigger depth corresponds to goals that require more planning) we can compute an estimator of the world with error $O(1/\sqrt{n})$. The amount of goals of depth $n$ grows doubly exponentially on $n$. We therefore improve the proof to only rely on depth-$n$ goals of \textit{width} 2 (where width measures how many different subgoals compose the main goal). In particular, we only require the agent to be capable of realizing a number of goals that grows singly exponential on $n$, and the estimator achieves an error of $O(\log n / n)$.

The rest of the paper is organized as follows. In Section~\ref{sec:basic_def} we present the main definitions from~\cite{richens2025general} (with a slightly different notation). In Section~\ref{sec:prev_res} we describe the original result alongside a summary of its proof. In Section~\ref{sec:new_res} we describe our results. Finally, in Section~\ref{sec:conclu} we detail some conclusions.  

\section{Basic definitions}\label{sec:basic_def}

We model the world as a Controlled Markov Decision Process.

\begin{definition}
    A \textbf{Controlled Markov Decision Process} (cMDP) is a 3-tuple $(\states, \actions, P)$ where $\states$ is a set of states, $\actions$ is a set of actions, and $P(\cdot | \cdot, \cdot):\states \times \actions \times \states \to [0,1]$ is a stochastic kernel satisfying $\sum_{s' \in \states} P(s'|s, a) = 1$ for all $(s, a) \in \states \times \actions$.
\end{definition}

Usually, a cMDP also includes a reward function and a discount factor, but we will not need them for the development of this work. We say that a cMDP $(\states, \actions, P)$ is \textbf{finite} if $|\states|,|\actions| < \infty$, and \textbf{communicating} if the directed graph $(\states, \{ss' : \exists a  \, \, P(s'|s, a) > 0\})$ is strongly connected. All cMDPs that we consider will be assumed to be both finite and communicating. Moreover, note that by definition the transition kernel does not depend on time (thus, we have defined \textit{stationary} cMDPs). 

\begin{example}\label{example:cMDP}
\normalfont
     Figure~\ref{fig:example_cMDP} shows an example of a finite cMDP with 5 states $\{s_{-2}, s_{-1}, s_0, s_1, s_2\}$ and two actions $\{L, R\}$ that correspond conceptually to left and right movements (except when the current state is the leftmost or rightmost, in which case it corresponds to a transition to $s_0$). Each edge corresponds to a non-zero transition probability: for example, the edge $s_0 \overset{R,p_R}{\longrightarrow} s_1$ indicates that $P(s_1|s_0,R) = p_R$. Observe that each action has a chance to ``fail'' and in that case no transition is performed. This is represented by the self-loops on each node.
    
    The cMDP is communicating under the assumption that $p_R,p_L > 0$. For example, the sequence of actions $RRRR$ has positive probability of transitioning state $s_{-2}$ to $s_2$: the exact value is $p_R^4$.
\end{example}

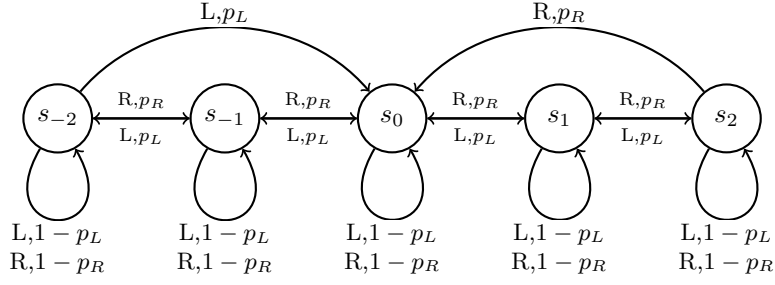
\begin{figure}[ht]
    \centering
        \begin{tikzpicture}[
            node/.style={circle, draw, thick, minimum size=9mm},
            edge/.style={->, thick},
            lab/.style={font=\small, fill=white, inner sep=1pt},
            scale=1.1
        ]
        
        \node[node] (sm2) at (-4,0) {$s_{-2}$};
        \node[node] (sm1) at (-2,0) {$s_{-1}$};
        \node[node] (s0)  at (0,0)  {$s_{0}$};
        \node[node] (s1)  at (2,0)  {$s_{1}$};
        \node[node] (s2)  at (4,0)  {$s_{2}$};
        
        \draw[edge] (sm2) -- (sm1) node[midway, above] {\scriptsize R,$p_R$};
        \draw[edge] (sm1) -- (s0)  node[midway, above] {\scriptsize R,$p_R$};
        \draw[edge] (s0)  -- (s1)  node[midway, above] {\scriptsize R,$p_R$};
        \draw[edge] (s1)  -- (s2)  node[midway, above] {\scriptsize R,$p_R$};
        
        \draw[edge] (sm1) -- (sm2) node[midway, below] {\scriptsize L,$p_L$};
        \draw[edge] (s0)  -- (sm1) node[midway, below] {\scriptsize L,$p_L$};
        \draw[edge] (s1)  -- (s0)  node[midway, below] {\scriptsize L,$p_L$};
        \draw[edge] (s2)  -- (s1)  node[midway, below] {\scriptsize L,$p_L$};
        
        \draw[edge, bend left=50] (sm2) to node[midway, above, lab] {L,$p_L$} (s0);
        \draw[edge, bend right=50] (s2)  to node[midway, above, lab] {R,$p_R$} (s0);
        
        \draw[edge, looseness=8] (sm2) to[out=-120,in=-60] node[below, lab] {\shortstack{L,$1-p_L$\\R,$1-p_R$}} (sm2);
        \draw[edge, looseness=8] (sm1) to[out=-120,in=-60] node[below, lab] {\shortstack{L,$1-p_L$\\R,$1-p_R$}} (sm1);
        \draw[edge, looseness=8] (s0)  to[out=-120,in=-60] node[below, lab] {\shortstack{L,$1-p_L$\\R,$1-p_R$}} (s0);
        \draw[edge, looseness=8] (s1)  to[out=-120,in=-60] node[below, lab] {\shortstack{L,$1-p_L$\\R,$1-p_R$}} (s1);
        \draw[edge, looseness=8] (s2)  to[out=-120,in=-60] node[below, lab] {\shortstack{L,$1-p_L$\\R,$1-p_R$}} (s2);
    
    \end{tikzpicture}
    
    \caption{Example of a cMDP with states $\states = \{s_{-2}, s_{-1}, s_0, s_1, s_2\}$ and $\actions = \{L, R\}$ where $p_R,p_L \in (0,1)$ are arbitrary probabilities.}
    \label{fig:example_cMDP}
\end{figure}

\newcommand{\history}{\mathbb{H}}

From now on, we fix a set of states $\states$ and actions $\actions$. A \textbf{history} is an infinite sequence $(s_0,a_0),(s_1,a_1),\ldots$ of state-action pairs, and we denote by $\history$ the set of all histories. For simplicity, we may write a history as the sequence $s_0a_0s_1a_1\ldots$ ignoring the pairing of states and actions. We use a fragment of LTL (Linear Temporal Logic) to describe subsets of desired histories. We refer to these formulas as \textit{goals}.

\begin{definition}
    A \textbf{basic goal} is a formula of the form $\mathcal{O}[V]$ where $\mathcal{O} \in \{\top,\bigcirc,\Diamond\}$ and $V \subseteq \states \times \actions$. \textbf{Sequential goals} are given by finite sequences of basic goals as
    \[
    \langle \varphi_1, \varphi_2, \ldots, \varphi_n\rangle
    \]
    where each $\varphi_i$ is a basic goal. Finally, (general) \textbf{goals} are a disjunction of sequential goals, as
    \[
    \bigvee_{i \in I} \psi_i
    \]
    where each $\psi_i$ is a sequential goal and $|I| < \infty$. We denote by $\Psi$ the set of all goals.
\end{definition}

We will employ some minor abuses of notation to resemble the notation of random variables. For example, we will write the goal $\bigcirc[\{s\} \times A]$ as $\bigcirc[\states=s]$, and similarly we write $\Diamond[A \neq a]$ for goal $\Diamond[\states \times (\actions \setminus \{a\})]$.

Whenever a history $h \in \history$ satisfies some formula $\varphi$, we denote it as $h\models \varphi$. The semantics for the basic goals given a history $h=s_0a_0s_1a_1\ldots $ are:
\begin{align*}
    h \models \top[V] & \iff (s_0,a_0) \in V\\
    h \models \bigcirc[V] & \iff (s_1,a_1) \in V\\
    h \models \Diamond[V] & \iff \exists i \in \mathbb{N} :(s_i,a_i) \in V
\end{align*}
Intuitively, the $\top$ operator specifies how the history should start, the $\bigcirc$ operator indicates what should be the next state-action pair and $\Diamond$ asks that the history must eventually visit some configuration.

For sequential goals, the semantics are described on a case-by-case basis:
\begin{align*}
    h = s_0 a_0 \ldots \models \langle \varphi_1,\ldots, \varphi_n\rangle \iff \begin{cases}
        (s_0,a_0) \in V \, \wedge h \models \langle \varphi_2,\ldots,\varphi_n\rangle & \varphi_1 = \top[V]\\
        (s_1,a_1) \in V \, \wedge s_1a_1s_2a_2\ldots \models \langle \varphi_2,\ldots,\varphi_n\rangle & \varphi_1 = \bigcirc[V]\\
        (s_i,a_i) \in V \,\wedge s_ia_i\ldots \models \langle \varphi_2,\ldots,\varphi_n\rangle & \varphi_1 = \Diamond[V]
    \end{cases}
\end{align*}
where in the case of the $\Diamond$ operator $i$ is the least index such that $(s_i,a_i) \in V$ (if there is no such index, then the goal is not satisfied). Finally, disjunction is defined naturally as $h \models  \psi_1 \vee \psi_2 \iff h \models \psi_1 \vee h \models \psi_2$.

The \textbf{depth} of a goal is defined as the length of the longest sequential goal appearing in it. Formally,
\begin{align*}
    &depth(\langle \varphi_1,\ldots,\varphi_n\rangle) = n\\
    &depth\left(\bigvee_{i \in I} \psi_i\right) = \max_{i \in I}\{depth(\psi_i)\}
\end{align*}
\newcommand{\finiteHistories}{\mathbb{H}_{\mathcal{F}}}
We denote by $\Psi_n$ the set of formulas of depth $n$.

\begin{example}\label{example:goals}
    \normalfont
    Consider the cMDP of Example~\ref{example:cMDP}. The history $s_0 R s_1 L s_0 R s_1 L s_0\ldots $ satisfies the basic goal $\top[\states = s_0]$ (the history starts at $s_0$) as well as $\bigcirc[\actions = L]$ (the second action taken is $L$) and $\Diamond[\{(s_0, R)\}]$ (the state-action pair $(s_0, R)$ is eventually visited). It does not satisfy the basic goal $\Diamond[\states = s_{-1}]$.
    
    Regarding sequential goals, it satisfies $\langle \top[\actions = R], \Diamond[\states = s_1], \bigcirc[\states = s_0]\rangle$ which states that the history's first action must be $R$, and then that it must eventually reach state $s_1$ and at the next step visit state $s_0$. It does not satisfy $\langle \Diamond[\actions = L], \bigcirc[\actions = L] \rangle$ because the first time it performs action $L$ is at the second step and an the third step the action taken is $R$. 
\end{example}

A finite history is a finite sequence $s_0a_0\ldots s_{n-1}a_{n-1} s_{n}$ of states and actions that starts and ends with a state. We denote the set of finite histories as $\finiteHistories$. A \textbf{policy} is a map $\pi: \finiteHistories \to \Delta(\actions)$, where $\Delta(\actions)$ denotes the set of probability distributions with support on $\actions$. Intuitively, a policy indicates, for each finite history, the next action that the agent wants to take, or, more in general, a distribution over the next possible action. We say that a policy is \textbf{deterministic} if for every $h \in \finiteHistories$ there is some $a \in \actions$ such that $\Pr\left(\pi(h) =a\right)=1$ (i.e. the distribution is concentrated in a single action for every finite history). We say that a policy is \textbf{markovian} if its distribution for the next action only depends on the last visited state. More formally, a policy is markovian if for every $h = s_0a_0\ldots s$ and $h' = s_0'a_0'\ldots s$ it holds that $\pi(h) = \pi(h')$.

A \textbf{goal-conditioned} policy is a mapping $\pi: \finiteHistories \times \Psi \to \Delta(\actions)$ such that for each $\psi \in \Psi$ the mapping $\pi(\cdot, \psi)$ is a policy. Goal-conditioned policies extend policies by allowing them to take different actions according to the requested goal. The notation $\pi_{\psi}$ will come in handy to represent the policy that is obtained by conditioning $\pi$ on $\psi$, i.e. $\pi_{\psi}(h) = \pi(h, \psi)$.

\begin{example}\label{example:policy}
    \normalfont
    Consider again the cMDP from Example~\ref{example:cMDP}. The policy $\pi_1$ given by
    \begin{align*}
        &\Pr\left(\pi_1(s_0a_0\ldots s_{i}) = R\right) = 1&\text{if } i \in \{-2,-1,0\}
        \\
        &\Pr\left(\pi_1(s_0a_0\ldots s_{i}) = L\right)  = 1&\text{if } i \in \{1,2\}
    \end{align*}
    attempts to always stay at $s_0$. It is deterministic and markovian. Another policy that also attempts to stay at $s_0$ is $\pi_2$ given by
    \begin{align*}
        &\Pr\left(\pi_2(s_0a_0\ldots s_{i}) = R\right) = 1&\text{if } i \in \{-1,0, 2\}
        \\
        &\Pr\left(\pi_2(s_0a_0\ldots s_{i}) = L\right)  = 1&\text{if } i \in \{-2, 1\}
    \end{align*}
    This policy tries to exploit the edges going directly from $s_{-2}$ and $s_2$ to $s_0$.
\end{example}

Given a cMDP with transition kernel $P$, a policy $\pi$ and an initial state $s$ we assign a probability to every finite history $h  = s_0a_0\ldots s_n' \in \finiteHistories$ as
\begin{align*}
    \Pr\left(h|\pi, s_0\right) =\begin{cases}
        \prod_{i=0}^{n-1} \Pr\left(\pi(h_i) = a_i\right) P(s_{i+1}|s_i, a_i) & s = s_0\\
        0 & \text{otherwise}
    \end{cases}
\end{align*}
where $h_i = s_0a_0 \ldots s_i$. Note that this induces a measure over $\finiteHistories$ for each $\pi$ and $s$, and we can extend this measure to $\history$ in the standard way via cylinder sets.

\begin{example}\label{example:probabilities}
    \normalfont
    Consider the policies $\pi_1$ and $\pi_2$ from Example~\ref{example:policy} and the finite history $h= s_0 R s_1 L s_0$. It holds that
    \begin{align*}
        \Pr\left( h |\pi_1, s_0\right) = \Pr\left( h |\pi_2, s_0\right) = p_Rp_L
    \end{align*}
\end{example}

We define the probability of a goal $\psi$ being achieved by a policy $\pi$ from state $s$ via induced measures as
\begin{align*}
    \Pr\left(\psi | \pi, s \right) = \sum_{\substack{h \in \history\\h \models \psi}} \Pr\left(h | \pi, s\right)
\end{align*}

We say that two goals $\varphi_1$ and $\varphi_2$ are \textbf{incompatible} if $h \models \varphi_1$ implies that $h \not\models \varphi_2$, and similarly that $h \models \varphi_2 \implies h \not\models \varphi_1$. We will use the following property regarding incompatible goals.
\begin{lemma}\label{lemma:disjoint_goals_sum_probs}
    Let $\varphi_1$ and $\varphi_2$ be two incompatible goals. Then,
    \begin{align*}
         \Pr\left(\varphi_1\vee \varphi_2 | \pi, s\right) = \Pr\left(\varphi_1 | \pi, s\right) +  \Pr\left(\varphi_2 | \pi, s\right)
    \end{align*}

    for any $s\in\states$. In particular, this implies that
    \begin{align*}
        \max_{\pi} \Pr\left( \varphi_1\vee \varphi_2|\pi,s \right) \leq \max_{\pi} \Pr\left(\varphi_1 | \pi,s\right) + \max_{\pi} \Pr\left(\varphi_2 | \pi,s\right)
    \end{align*}
\end{lemma}

\begin{proof}
    Fix some policy $\pi$ and note that:
    \begin{align*}
        \Pr\left(\varphi_1\vee \varphi_2 | \pi,s\right) &= \sum_{\substack{h \in \history\\h\models \varphi_1\vee \varphi_2}} \Pr\left( h|\pi,s\right)\\
        &= \sum_{\substack{h \in \history\\h\models \varphi_1}} \Pr\left( h|\pi,s\right) + \sum_{\substack{h \in \history\\h\models  \varphi_2}} \Pr\left( h|\pi,s\right)\\
        &= \Pr\left(\varphi_1 | \pi,s\right) + \Pr\left(\varphi_2| \pi,s\right)
    \end{align*}
    where we used the hypothesis over $\varphi_1$ and $\varphi_2$ in the second equality.
\end{proof}

\newcommand{\repeatpair}{R}

%
%

The next definitions aim to capture the set of policies that have some capacity for general reasoning.
\begin{definition}
    We say that a goal-conditioned policy $\pi^*$ is \textbf{$\delta$-optimal} for some subset of goals $\Phi \subseteq \Psi$ if for all $\varphi \in \Phi$ and $s\in \states$
    \begin{align*}
        (1-\delta)\max_{\pi} \{\Pr\left( \varphi | \pi, s \right)\} \leq \Pr\left( \varphi | \pi_{\varphi}^*, s\right)
    \end{align*}

    An agent is \textbf{$n$-th-degree $\delta$-optimal} if it is $\delta$-optimal for $\Psi_n$.
\end{definition}

\begin{example}\label{example:optimal-policy}
    \normalfont
    Consider again the policies $\pi_1$ and $\pi_2$ from Example~\ref{example:policy}, and the goal $\varphi = \Diamond[\states = s_0]$. It holds that
    \begin{align*}
        \Pr\left(\varphi | \pi_1, s\right) = \Pr\left(\varphi | \pi_2, s\right) = 1    
    \end{align*}
    for every starting state $s$.

    Now, consider the goal $\psi = \langle \Diamond[\states \in \{s_{-2}, s_2\}], \bigcirc[\states = s_0]]\rangle$ which states that the agent must reach one of the states $s_{-2}$ or $s_2$ and then go to $s_0$. We can consider two policies $\pi_3$ and $\pi_4$ that attempt to do that as
    \begin{align*}
    \Pr\left(\pi_3(\cdot) = L\right) = 1\\
    \Pr\left(\pi_4(\cdot) = R\right) = 1
    \end{align*}
    where $\pi_3$ always moves left, while $\pi_4$ always moves right. It can be seen that
    \begin{align*}
        \Pr\left(\psi | \pi_3, s_0\right) = p_L\\
        \Pr\left(\psi  | \pi_4, s_0 \right) = p_R
    \end{align*}
    Therefore, depending on the values of $p_L$ and $p_R$, one may be better than the other. Moreover, it can be shown that one of these policies is optimal for $\psi$ (i.e. if $p_L \geq p_R$ then $\pi_3$ is optimal and $\pi_4$ is optimal in the other case).

    Therefore, for $\Phi = \{\varphi, \psi\}$ it holds that the goal-conditioned policy $\pi$ given by $\pi_{\varphi} = \pi_1$ and $\pi_{\psi} = \pi_3$ is optimal for $\Phi$ if $p_L \geq p_R$. Furthermore, if $p_L < p_R$ but $(1-\delta)p_R \leq p_L$ for some $\delta \in (0, 1)$ it holds that $\pi$ is $\delta$-optimal for $\Phi$.
\end{example}

\section{Previous result}\label{sec:prev_res}

For the remainder of the paper, we fix an initial state $s_0$. In \cite{richens2025general}[Theorem 1] the following results is shown:

\begin{theorem}\label{teo:deterministic_policies_induce_transitions}
    Let $(\states, \actions, P)$ be a cMDP and $\pi$ an $n$-th-degree $\delta$-optimal goal-conditioned deterministic policy. Then, for any $s,s' \in \states$ and $a \in \actions$ it is possible to define an estimator $\hat{P}(s'|s, a)$ by querying the policy $\pi$ such that
    \begin{align*}
        |\hat{P}(s'|s,a) -P(s'|s,a)| \leq \sqrt{\frac{2 P(s'|s,a)(1-P(s'|s,a))}{(n-1)(1-\delta)}}
    \end{align*}
\end{theorem}

This result is obtained by considering a family of goals that encode information on $P(s'|s,a)$. Given a binary sequence $w\in \{0,1\}^n$, consider the goal
\begin{align*}
    \varphi_{b,w} = \langle \varphi_1^b, \varphi_2, \varphi_3^{w_1},\varphi_2,\varphi_3^{w_2},\ldots, \varphi_2,\varphi_3^{w_n}\rangle
\end{align*}
where $\varphi_1^b = \top[\actions=b]$, $\varphi_2 = \Diamond[\states=s, \actions=a]$, $\varphi_3^0 = \bigcirc [S\neq s']$ and $\varphi_3^{1} = \bigcirc[\states= s']$. Intuitively, this goal states that the first action taken must be $b$, that the state-action pair $(s, a)$ must be visited $n$ times and that at the $i$-th visit the next state must be $s'$ if $w_i = 1$ and different from $s'$ if $w_i=0$. Given this goal, we can build
\begin{align}\label{eq:def_of_rho}
    \rho_{b, r} = \bigvee_{\substack{w\in\{0,1\}^n\\ |w|_1 = r}} \varphi_{b, w}
\end{align}
where $|w|_1 = |\{ i : w_i = 1\}|$. This goal enforces that the first action taken must be $b$, and that in the next $n$ visits of the state-action pair $(s,a)$ it must be the case that in $r$ of them the next visited state is $s'$. The following lemma is a key component of the proof:

\begin{lemma}[Proven in Lemma 6 of \cite{richens2025general}]\label{lemma:prob_of_rho_depends_on_p}
    There exists a policy $\pi^*$ such that, for every $b$ and $w \in \{0,1\}^n$, it holds that
    \begin{align}\label{eq:prob_phi_w}
        \max_{\pi}\Pr\left( \varphi_{b,w}|\pi, s_0 \right) = \Pr\left(\varphi_{b,w}|\pi^*, s_0\right) = P(s'|s,a)^{|w|_1}(1-P(s'|s, a))^{n-|w|_1}
    \end{align}
    In particular, this implies that
    \begin{align*}
        \max_{\pi} \Pr\left( \rho_{b, r}|\pi, s_0 \right) = \Pr\left( \rho_{b,r} | \pi^*, s_0 \right) = \binom{n}{r}P(s'|s,a)^r(1-P(s'|s,a))^{n-r} 
    \end{align*}
\end{lemma}

\begin{proof}
    For the first part of the statement, we will only describe informally the policy proposed in the original proof. The second part of the statement follows directly from the first one and Lemma~\ref{lemma:disjoint_goals_sum_probs}.

    The policy $\pi^*$ starts by taking action $b$ in the first step. Then, for every subsequent step it always takes an action that gets it ``closer'' to $s$, on which it will always perform action $a$. It can be proven that, because the environment is finite and communicating, there exists a deterministic way of picking the actions of $\pi^*$ in order to guarantee that the state $s$ will be reached with probability 1 (see Lemma 1 from \cite{richens2025general}). Note that once the policy $\pi^*$ reaches the state-action pair $(s, a)$ with probability $P(s'|s,a)$ the next visited state is $s'$. Therefore, by looking at the number of $1$s in $w$ we can compute the probability of success of $\pi^*$ and obtain the value from Eq.~\eqref{eq:prob_phi_w}.
\end{proof}

The policy described in the proof of Lemma~\ref{lemma:prob_of_rho_depends_on_p} is shown conceptually in Figure~\ref{fig:proof_strat}. Each time the policy reaches state $s$ it transitions with probability $p = P(s'|s, a)$ to $s'$. Then, it returns to $s$ with probability 1.

\begin{figure}[ht]
    \centering
    \includegraphics[width=0.5\linewidth]{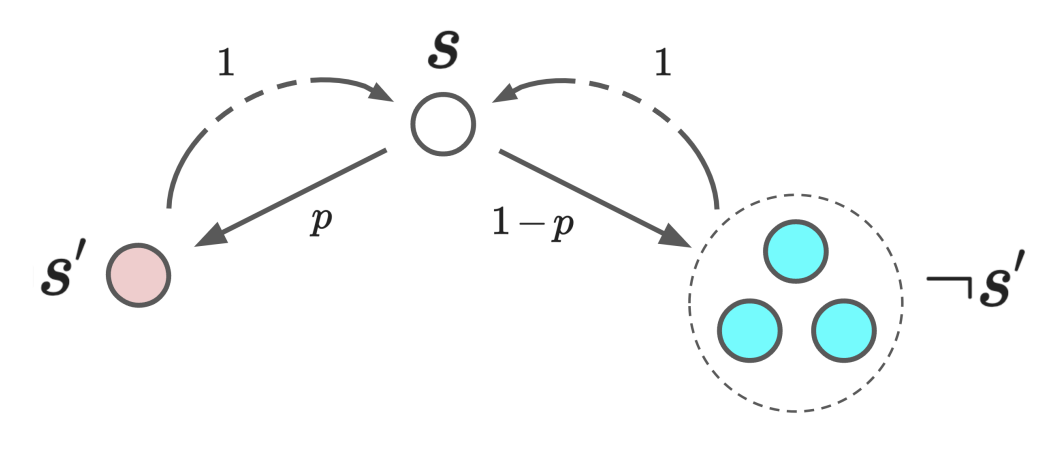}
    \caption{Diagram representing conceptually the policy $\pi^\star$ described in the proof of Lemma~\ref{lemma:prob_of_rho_depends_on_p} with $p=P(s'|s,a)$. This figure was taken from \cite{richens2025general}.}
    \label{fig:proof_strat}
\end{figure}

Two other key goals are considered:

\begin{align}
    &\psi_{b,k} = \bigvee_{r \leq k} \rho_{b, r} &\chi_{b,k} = \bigvee_{r > k} \rho_{b, r} \label{eq:def_chi_psi}
\end{align}

We can also compute their probability of being achieved:

\begin{lemma}\label{lemma:prob_of_psi_chi_depend_on_p}
    Let $\psi_{b, k}$ and $\chi_{b, k}$ be defined as in Eq.~\eqref{eq:def_chi_psi}. Then,
    \begin{align*}
        \max_{\pi} \Pr\left( \psi_{b, k}|\pi, s_0\right) = \sum_{r \leq k} \binom{n}{r} P(s'|s,a)^r(1-P(s'|s,a))^{n-r}\\
        \max_{\pi} \Pr\left( \chi_{b, k}|\pi, s_0\right) = \sum_{r > k} \binom{n}{r} P(s'|s,a)^r(1-P(s'|s,a))^{n-r}
    \end{align*}
\end{lemma}

\begin{proof}
    This is an immediate consequence of Lemmas~\ref{lemma:disjoint_goals_sum_probs} and~\ref{lemma:prob_of_rho_depends_on_p}
\end{proof}

\section{New results}\label{sec:new_res}

\subsection{The case of stochastic policies} 

We extend Theorem~\ref{teo:deterministic_policies_induce_transitions} for stochastic policies. This extension comes at the cost of requiring $\delta < \frac{1}{2}$ and a slower rate of convergence.

\begin{theorem}\label{teo:non-deterministic-policies-induce-transitions}
    Let $(\states, \actions, P)$ be a cMDP, $\pi$ an $(2n+1)$-th-degree $\delta$-optimal goal-conditioned policy with $\delta < \frac{1}{2}$, and $L = \log\left( \frac{2(1-\delta)}{1-2\delta} \right)$. Then, for any $s,s' \in \states$ and $a \in \actions$ it is possible to define an estimator $\hat{P}(s'|s, a)$ by querying the policy $\pi$ such that
    \begin{align*}
        |\hat{P}(s'|s,a) -P(s'|s,a)| \leq \sqrt{\frac{2P(s'|s,a)(1-P(s'|s,a))}{n}L} + \frac{2L}{3n} + \frac{1}{n}
    \end{align*}

    In particular, for any fixed $\delta$ the error of the estimator goes to zero as $n$ goes to infinity at rate $O(1/\sqrt{n})$.
\end{theorem}

\begin{proof}
    Consider the goals $\rho_{b, r}$, $\psi_{b,k}$ and $\chi_{b,k}$ defined in Eqs.~\eqref{eq:def_of_rho} and~\eqref{eq:def_chi_psi}, each of which has depth $2n+1$.

    If $X = Bin(P(s'|s, a), n)$ is a binomial random variable, by Lemmas~\ref{lemma:prob_of_rho_depends_on_p} and~\ref{lemma:prob_of_psi_chi_depend_on_p}
    \begin{align*}
        \Pr\left(X=r\right) &= \max_{\pi}\Pr\left(\rho_{b, r}|\pi, s_0\right)\\
        \Pr\left(X \leq k\right) &= \max_{\pi} \Pr\left(\psi_{b, k}|\pi, s_0\right)\\
        \Pr\left(X > k\right) &= \max_{\pi} \Pr\left(\chi_{b, k}|\pi, s_0\right)
    \end{align*}

    Let $\xi_{k} = \psi_{a,k} \vee \chi_{b,k}$ with $a\neq b$. Note that
    \begin{align*}
        \max\{\Pr\left(X\leq k\right), \Pr\left(X > k\right)\} = \max_{\pi} \Pr\left(\xi_{k}| \pi,s_0\right)
    \end{align*}
    because on the first action the policy has to decide whether to aim for $\psi_{a, k}$ or $\chi_{b, k}$ by taking respectively action $a$ or $b$, and it is optimal to put all the probability in the easiest goal to achieve.

    Let $p_{c,k} = \Pr(\pi_{\xi_k}(s_0) = c)$ for $c \in \{a, b\}$. By $\delta$-optimality of $\pi$, we have
    \begin{align*}
        (1-\delta)\max(\Pr(X\leq k), \Pr(X > k)) \leq p_{a, k} \Pr(X \leq k) + p_{b,k} \Pr(X > k)
    \end{align*}

    Suppose that $\Pr(X \leq k) = \frac{1}{2} + \varepsilon$ for some $\varepsilon\geq 0$. Then, because $p_{a,k} + p_{b,k} \leq 1$, it holds that
    \begin{align*}
        (1-\delta)\Pr(X \leq k) &\leq p_{a,k} \Pr(X\leq k) + p_{b,k}\Pr(X > k)\\
        &\leq p_{a, k} \Pr(X \leq k) + (1-p_{a,k})(1-\Pr(X \leq k))
    \end{align*}
    which implies, letting $F = \Pr(X \leq k)$, the inequality
    \begin{align*}
        p_{a,k} \geq \frac{(2-\delta)F-1}{2F-1}
    \end{align*}
    and replacing $F=\frac{1}{2}+\varepsilon$ we obtain
    \begin{align}\label{eq:lower_bound_probabilities}
        p_{a,k} \geq 1 - \frac{\delta}{2} - \frac{\delta}{4\varepsilon}
    \end{align}
    We can deduce an analogous inequality for $p_{b,k}$ when $\Pr(X> k) = \frac{1}{2} + \varepsilon$.

    These equations constrain the probabilities $p_{a,k}$ and $p_{b,k}$, implying that the agent must focus more on either $\psi_{a, k}$ (i.e. $p_{a,k} > p_{b,k}$) or $\chi_{b, k}$ (i.e. $p_{a,k} < p_{b,k}$) whenever there is a big gap between $\Pr(X \leq k)$ and $\Pr(X > k)$. For example, when $k=n$ it is the case that $p_{a,k} \geq 1-\delta > \frac{1}{2}>p_{b,k}$ because $\Pr(X > k) = 0$. A similar argument applies when $k = -1$, and we conclude that $p_{b,-1}>p_{a,-1}$.

    In general, through Eq.~\eqref{eq:lower_bound_probabilities} we note that when $\varepsilon > \frac{\delta}{2(1-\delta)}$ the agent is forced to place more probability in action $a$, implying that $p_{a,k} > \frac{1}{2}$. Similarly, when $\Pr(X > k) = \frac{1}{2} + \varepsilon$ with $\varepsilon > \frac{\delta}{2(1-\delta)}$ we conclude that $p_{b, k} > \frac{1}{2}$. Note that since $\varepsilon \leq \frac{1}{2}$, if such a $\varepsilon$ exists it must be the case that $\frac{\delta}{2(1-\delta)}< \frac{1}{2}$, which implies $\delta < \frac{1}{2}$.

    Therefore, to estimate the mean of the distribution (which has the form $pn$ with $p=P(s'|s,a)$) we will look at the following index:
    \begin{align*}
        x &= \max \{i \in [-1, n] : p_{b, i} \geq p_{a,i}\}
    \end{align*}

   Let $\alpha = \frac{1-2\delta}{2(1-\delta)}$. We note that, because $p_{b,x} \geq p_{a,x}$, it can't be the case that $P(X> x) < \frac{1}{2} - \frac{\delta}{2(1-\delta)} = \alpha$, since that would imply that $p_{a,x} > p_{b, x}$ by our previous remarks. Thus, we conclude that $P(X > x) \geq \alpha$. Through a similar argument we observe that since $p_{a, x+1} > p_{b, x+1}$ (remember that $x < n$ because $P(X \leq n) = 1$) it holds that $P(X\leq x+1) \geq \alpha$. See Figure~\ref{fig:proof_arg} for a visual explanation of this reasoning in a particular case.

    We will use the Bernstein-Freedman inequalities, which state that for any random variable $Z$ with mean $\mu$ and variance $\sigma^2$ and $t \geq 0$ it holds that
    \begin{align*}
        \Pr(Z-\mu \geq t) &\leq \exp{}\left(-\frac{t^2}{2\sigma^2 + \frac{2t}{3}}\right)\\
        \Pr(\mu-Z \geq t) &\leq \exp{}\left(-\frac{t^2}{2\sigma^2 + \frac{2t}{3}}\right)
    \end{align*}

    Let $L = \log{} \frac{1}{\alpha}$, which satisfies $L > 0$ because $0< \alpha < 1$. We set $t_{\star} = \sqrt{2\sigma^2L} + \frac{2}{3}L$, noting that $\frac{t_{\star}^2}{2\sigma^2 + \frac{2t_{\star}}{3}} > L$, which implies that $\Pr(X \geq \mu +  t_{\star}) < e^{-L} = \alpha$ and that $\Pr(X \leq \mu -t_{\star}) < \alpha$.

    If $x \geq \mu + t_{\star}$ we conclude that
    \begin{align*}
        \alpha \leq \Pr(X > x) \leq \Pr(X \geq \mu + t_{\star}) < \alpha
    \end{align*}
    which is absurd. Thus, it is the case that $x < \mu + t_{\star}$.

    Through an analogous argument we see that if $x+1 \leq\mu - t_{\star}$ then
    \begin{align*}
    1-\alpha \geq \Pr(X > x+1) \geq \Pr(X > \mu-t_{\star}) = 1-\Pr(X \leq \mu - t_{\star}) > 1-\alpha
    \end{align*}
    which is again absurd. Thus,
    \begin{align*}
        \mu - t_{\star}-1 \leq x \leq \mu + t_{\star}
    \end{align*}
    and therefore
    \begin{align*}
        \left|x - \mu\right| \leq t_{\star}+1
    \end{align*}

    Rewriting $t_{\star}$ and dividing by $n$, we conclude that
    \begin{align*}
        \left|\frac{x}{n} - p\right| \leq \sqrt{\frac{2p(1-p)}{n} \log\left(\frac{2(1-\delta)}{1-2\delta}\right)} + \frac{2}{3n}\log\left(\frac{2(1-\delta)}{1-2\delta}\right) + \frac{1}{n}
    \end{align*}
    which converges to zero as $n \to \infty$ for any $\delta$ with speed $O(1/\sqrt{n})$ because of the first term.
\end{proof}

\begin{figure}[!ht]
\captionsetup{width=0.8\textwidth}
    \centering
    \includegraphics[width=0.8\linewidth]{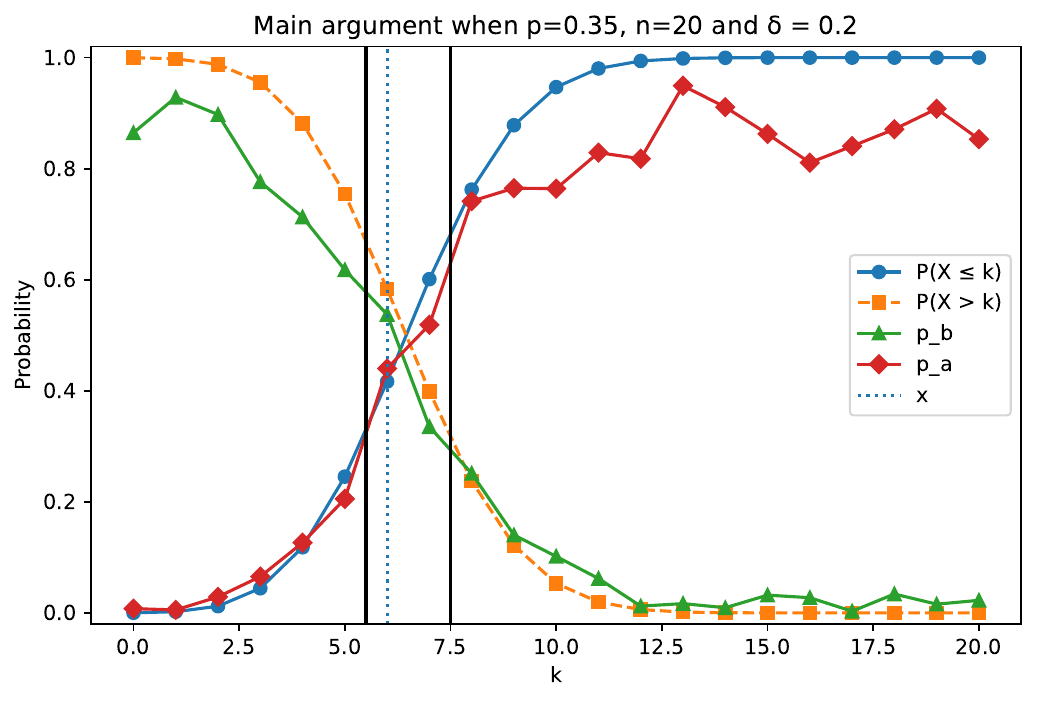}
    \caption{Diagram of the proof of Theorem~\ref{teo:non-deterministic-policies-induce-transitions} when $p=0.35$, $n=20$ and $\delta = 0.2$. The orange and blue lines (squares and circles) are the probabilities $P(X > k)$ and $P(X \leq k)$ for $k =0\ldots n$. The green and red lines (triangles and diamonds) represent feasible values of $p_{b,k}$ and $p_{a,k}$ for the different values of $k$ (they were chosen randomly). Note that our predictor $x$ would correspond to the vertical line at 6. In this case $\varepsilon = 0.125$, which means that for $k$ on the left of the black line at 5.5 we are guaranteed that $p_{b,k} \geq p_{a,k}$, and similarly, for $k \geq 8$ we know that $p_{a,k} \geq p_{b,k}$. Inside the black bars both options are possible, and thus the predictor $x$ could have turned out to be $5$ or $7$ instead.}
    \label{fig:proof_arg}
\end{figure}

We observe that the hypothesis $\delta < \frac{1}{2}$ is required only because the proof strategy relies on dichotomic goals, and therefore a $\frac{1}{2}$-optimal agent can set $p_{a,k} = p_{b,k} = \frac{1}{2}$ for all $k$ which would give no information about the distribution of $X$.

Theorem~\ref{teo:non-deterministic-policies-induce-transitions} is a natural extension of Theorem~\ref{teo:deterministic_policies_induce_transitions}, especially considering that the proof strategy is identical. We note that the dependence on $\delta$ is worse, and in general that the constants in the error bound are larger. We believe this is also reasonable: it is common for stochastic policies or algorithms to avoid a deeper understanding of the task they are trying to solve by taking randomized choices. In the introduction we mention the case of the Miller-Rabin test and the AKS algorithm. Another example is the problem of deciding whether a polynomial on many variables is identically 0: it is easy to solve by evaluating random points and invoking the Schwartz-Zippel Lemma \cite{schwartz1980fast}, while the best deterministic algorithms exploit deeper properties about polynomials and are less efficient.

\subsection{Partially observable environments}

We extend Theorem~\ref{teo:non-deterministic-policies-induce-transitions} to partially observable environments.

\begin{definition}
    A \textbf{partially observable cMDP} is a 5-tuple $(\states, \actions, P, \observations, \Omega)$ where $(\states, \actions, P)$ is a cMDP, $\observations$ is a set of observations and $\Omega(\cdot|\cdot): \states \times \observations \to [0,1]$ is a stochastic kernel satisfying $\sum_{o \in \observations} \Omega(o | s) = 1$ for all $s \in \states$.
\end{definition}

Partially observable cMDPs generalize cMDPs by introducing an ``observation'' kernel $\Omega$ that describes how the policy gets to know the current state.

\newcommand{\failObs}{\textsc{fail}}

\begin{example}\label{example:partially_obs_cMDP}
    \normalfont Let $(\states, \actions, P, \observations, \Omega)$ be a partially observable cMDP where $(\states, \actions, P)$ is the cMDP from Example~\ref{example:cMDP}. Let $\observations = \{s_{-2}, s_{-1}, s_0, s_1, s_2, \failObs\}$, where $\failObs$ corresponds conceptually to the situation when the policy fails to detect its current state. Define $\Omega$ as
    \begin{align*}
        \Omega(\failObs|\cdot) &= p_F\\
        \Omega(s|s) &= 1-p_F \text{ for every } s \in \states
    \end{align*}
    where $p_F \in (0, 1)$ represents the failure probability. Note that when the agent takes action $R$ at $s_0$ it fails to observe whether the action was successful (i.e. if the current state is $s_1$) with probability $p_F$. 
\end{example}

A partially observable cMDP is finite or communicating if its underlying cMDP is finite or communicating.

\newcommand{\historyObs}{\history_{\observations}}
\newcommand{\finiteHistoriesObs}{\mathbb{H}_{\observations, \mathcal{F}}}

A history of observations is an infinite sequence of observation-action pairs $o_0a_0o_1 a_1\ldots$, and we denote by $\history_{\observations}$ the set of all histories of observations. A finite history is a finite sequence of observation-action pairs that starts and ends with an observation, and we denote by $\history_{\observations, \mathcal{F}}$ the set of finite histories of observations. A \textbf{observation-based policy} (i.e. a policy for a partially observable cMDP) is a mapping $\pi:\finiteHistoriesObs \to \Delta(\actions)$. We define goal-conditioned observation-based policies in the same way as we did for fully observable cMDPs. The notions of markovian and deterministic policies are defined in the natural way following what we did in the fully observable case.

\begin{example}\label{example:observation_based_policy}
    \normalfont Consider the policy over the cMDP from Example~\ref{example:partially_obs_cMDP} defined as
    \begin{align*}
        \Pr\left(\pi(o_0 a_0\ldots o_i) = R\right)=1 \iff o_i \in \{s_{-2}, s_{-1}, s_0\}\\
        \Pr\left(\pi(o_0 a_0\ldots o_i) = L\right)=1 \iff o_i = \{s_1,s_2,\failObs\}
    \end{align*}
    
    This policy is deterministic and markovian, and tries to stay at $s_0$. Whenever the \failObs{} observation is recorded, the policy tries to move left.
\end{example}

Fix some policy $\pi$ and starting state $s$. Given a real (state) history $h = s_0a_0 \ldots s_n$ and a history of observations $h_\observations = o_0a_0 \ldots o_n$ we can compute the joint probability of $h$ and $h_\observations$ given $\pi$ and $s$ as
\begin{align*}
    \Pr\left(h,h_\observations |\pi, s\right) = \begin{cases}
        \prod_{i=0}^n \Omega(o_i|s_i) \prod_{i=0}^{n-1} \Pr\left(\pi((h_{\observations})_i) = a_i\right) P(s_{i+1}|s_i,a_i) & s_0 = s\\
        0 & \text{otherwise}
    \end{cases}
\end{align*}
where $(h_\observations)_i = o_0a_0\ldots o_i$. Then, we can compute the probability of a given real history as
\begin{align*}
    \Pr\left( h |\pi, s \right) =
        \sum_{h_\observations = o_0a_0\ldots o_n \in \finiteHistoriesObs} \Pr\left(h,h_\observations | \pi, s\right)
\end{align*}

Given some goal $\psi$ as the one previously described, we can straightforwardly define the probability of it being achieved as $\Pr\left(\psi | \pi, s\right) = \sum_{\substack{h \in \history \\ h \models \psi}} \Pr\left( h | \pi, s\right)$.

\begin{example}
    \normalfont Consider again the cMDP of Example~\ref{example:cMDP}. Let $\psi = \langle \Diamond [\{s_{-2}, s_{2}\}], \bigcirc[s= s_0]\rangle$ and consider the policy given by
    \begin{align*}
        \Pr\left(\pi(o_0 a_0 \ldots o_n) = L \right) = 1 &\iff o_n \in \{s_{-2}, s_{-1}, s_{0}, s_1, s_2\}\\
        \Pr\left(\pi(o_0 a_0 \ldots o_n) = R \right) = 1 &\iff o_n \in \{\failObs\}
    \end{align*}
    Policy $\pi$ attempts to fulfill $\psi$ by always moving left, except than when it fails at looking at the current state it moves right. It can be seen that
    \begin{align*}
        \Pr\left(\psi | \pi, s_0\right) = (1-p_F)p_L + p_Rp_F
    \end{align*}
    The reasoning is as follows: if the policy satisfies the first component of $\psi$ by arriving to $s_{-2}$, then it has a probability of $(1-p_F)p_L$ of then taking the edge $s_{-2} \overset{L}{\rightarrow} s_0$ succesfully. Otherwise, if it first arrived to $s_2$, then its probability of going to $s_0$ in the next step is $p_R p_F$ through the edge $s_{2} \overset{R}{\rightarrow} s_0$.
\end{example}

We observe that since the policy depends on observations and the goal is phrased in terms of real states, the policy can potentially achieve some goal without realizing it.

Our motivation for extending Theorem~\ref{teo:deterministic_policies_induce_transitions} into Theorem~\ref{teo:non-deterministic-policies-induce-transitions} before considering the partially observable case stems from the fact that stochasticity can be relevant for achieving goals when information is missing.

\begin{example}\label{example:randomness_in_part_obs_is_useful}
    \normalfont Consider the partially observable cMDP from Figure~\ref{fig:three-state-graph} and the goal $\varphi = \Diamond[\states = s_3]$. Note that any deterministic, markovian and observation-based policy $\pi$ will satisfy $\Pr\left(\varphi | \pi, s_1\right) = 0$, since to arrive at $s_3$ actions $a$ and $b$ must be taken, and meanwhile any policy of this kind will always perform either $a$ or $b$ when on states $s_1$ and $s_2$ because they both provide the same observation. By contrast, a fully stochastic uniform policy (i.e. $\Pr\left( \pi(h) = a \right) = \frac{1}{|\actions|}$ for every $a \in \actions$) will reach $s_3$ with probability one. Note that this randomized policy is markovian.
\end{example}

\begin{figure}[ht]
  \centering
  \begin{tikzpicture}[
      node/.style={circle, draw, thick, minimum size=9mm},
      edge/.style={->, thick},
      lab/.style={font=\small, fill=white, inner sep=1pt},
      grouplab/.style={font=\small, fill=white, inner sep=1pt},
      scale=1.1
  ]

  \node[node] (s1) at (-3,0) {$s_1$};
  \node[node] (s2) at (0,0) {$s_2$};
  \node[node] (s3) at (3,0) {$s_3$};

  \draw[thick, dashed] (-1.5,0) circle (2.4);
  \node[grouplab] at (0.2,2.1) {$o_1$};

  \draw[thick, dashed] (3,0) circle (0.95);
  \node[grouplab] at (3.8,0.85) {$o_2$};

  \draw[edge, bend left=15] (s1) to node[midway, above, lab] {$a$} (s2);
  \draw[edge, bend left=15] (s2) to node[midway, above, lab] {$a$} (s1);
  \draw[edge, bend left=30] (s3) to node[midway, above, lab] {$a$} (s1);

  \draw[edge] (s2) -- (s3) node[midway, above] {\scriptsize $b$};
  \draw[edge, looseness=8] (s1) to[out=-120,in=-60] node[below, lab] {$b$} (s1);
  \draw[edge, looseness=8] (s3) to[out=-120,in=-60] node[below, lab] {$b$} (s3);

  \end{tikzpicture}
  \caption{A partially observable cMDP with three states and two actions. An edge $s \overset{a}{\rightarrow} s'$ indicates that $P(s'|s,a) = 1$. Both states $s_1$ and $s_2$ are related to a unique observation $o_1$ (i.e. $\Omega(o_1|s_1) = \Omega(o_2,s_2) = 1$). State $s_3$ has its own observation $o_2$.}
  \label{fig:three-state-graph}
\end{figure}
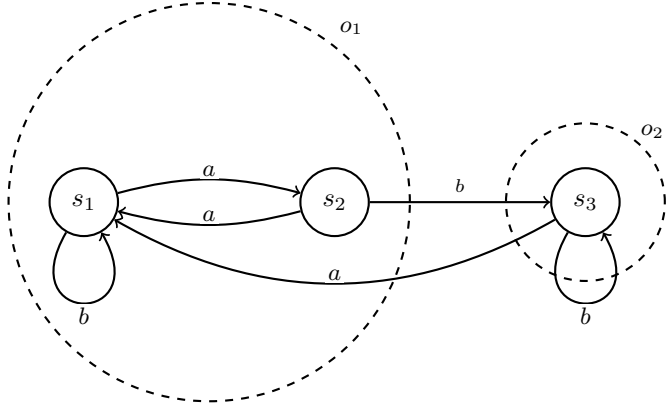

Observe that the notion of an $n$th degree $\delta$-optimal agent still makes sense in partially observable worlds. Then, we prove the following statement.

\begin{theorem}\label{teo:non-deterministic-policies-induce-transitions-in-partial-obs}
    Theorem~\ref{teo:non-deterministic-policies-induce-transitions} continues to hold even if the cMDP is partially observable and $\pi$ is a goal-conditioned observation-based policy.
\end{theorem}

\begin{proof}
    The proof from Theorem~\ref{teo:non-deterministic-policies-induce-transitions} only depends on the fact that the probability of success for optimal policies on the proposed goals can be computed analytically and depends on the relevant probability as per Lemmas~\ref{lemma:prob_of_rho_depends_on_p} and~\ref{lemma:prob_of_psi_chi_depend_on_p}. We show that this also holds for the partially observable case.
    
    Following the notation from Lemma~\ref{lemma:prob_of_rho_depends_on_p}, let us define, for any $w \in \{0,1\}^n$, the goal
    \begin{align*}
        \varphi_{w} = \langle \varphi_2, \varphi_3^{w_1}, \ldots, \varphi_2, \varphi_3^{w_n} \rangle
    \end{align*}
    Let $\pi$ be a policy that always moves at random, i.e.
    \begin{align*}
        \Pr\left(\pi(o_0a_0 \ldots o_n) = b'\right) &= \frac{1}{|\actions|} \, \forall b' \in \actions
    \end{align*}
    If we show that for any $s_0 \in \states$ and $w \in \{0, 1\}$ it holds that
    \begin{align*}
        \Pr\left( \varphi_{w} | \pi, s_0\right) = p^{|w|_1}(1-p)^{n-|w|_1}
    \end{align*}
    where $p = P(s'|s,a)$ we will have finished, since Lemma~\ref{lemma:disjoint_goals_sum_probs} still holds in the partially observable case, and from these facts it is easy to observe that the policy $\pi'$ that first takes action $b$ and then acts like $\pi$ satisfies $\Pr\left( \rho_{r, n} |\pi, s_0\right) = \binom{n}{r} p^r (1-p)^{n-r}$.\footnote{Note that this policy $\pi'$ is neither deterministic nor markovian.}

    We proceed by induction on $n$. The result is trivially true for $n=0$. For the inductive case take $w \in \{0, 1\}^{n+1}$, and assume that $w =1 w'$ (the other case can be proven analogously). Note that, for any $s_0 \in \states$, 
    \begin{align}
        \Pr\left(\varphi_w | \pi, s_0 \right) &= \sum_{\substack{h \models \varphi_w}} \Pr\left(h |\pi, s_0\right)\nonumber\\
        &= \sum_{\substack{h \models \varphi_w}} \sum_{h_\observations \in \historyObs} \Pr\left(h, h_\observations |\pi, s_0\right)\nonumber\\
        &= \sum_{\substack{h \models \varphi_w \\ h=s_0a_0\ldots s_n}} \sum_{\substack{h_\observations \in \finiteHistoriesObs\\h_\observations= o_0a_0 \ldots o_n}} \prod_{i=0}^n \Omega(o_i|s_i) \prod_{i=0}^{n-1} \Pr\left(\pi((h_{\observations})_i) = a_i\right) P(s_{i+1}|s_i,a_i) \nonumber\\
        &=\sum_{\substack{h \models \varphi_w \\ h=s_0a_0\ldots s_n}} \sum_{\substack{h_\observations \in \finiteHistoriesObs\\h_\observations= o_0a_0 \ldots o_n}} \prod_{i=0}^n \Omega(o_i|s_i) \prod_{i=0}^{n-1} \frac{P(s_{i+1}|s_i,a_i)}{|\actions|} \nonumber\\
        &= \sum_{\substack{h \models \varphi_w \\ h=s_0a_0\ldots s_n}} \prod_{i=0}^{n-1} \frac{P(s_{i+1}|s_i,a_i)}{|\actions|} \sum_{\substack{h_\observations \in \finiteHistoriesObs\\h_\observations= o_0a_0 \ldots o_n}} \prod_{i=0}^n \Omega(o_i|s_i)\nonumber \\
        &= \sum_{\substack{h \models \varphi_w \\ h=s_0a_0\ldots s_n}} \prod_{i=0}^{n-1} \frac{P(s_{i+1}|s_i,a_i)}{|\actions|} \label{eq:rewriting_of_success}
    \end{align}
    In particular, we have shown that if the policy does not depend on the observations then the kernel $\Omega$ of partial observability does not have any impact on the computations of the probabilities. Note that during the computation when summing over the histories $h$ such that $h \models \varphi_w$ we have ``discretized'' the computations by taking a set of finite prefixes that satisfy $\varphi_w$. This is well-defined because the satisfaction of goals actually depends on finite prefixes.

    Now, let $\overline{P}(s_{i+1}| s_i, a_i) = \frac{P(s_{i+1}|s_i, a_i)}{|\actions|}$. If $h \models \varphi_w$ then $h = s_0 a_0 \ldots sas'\ldots$, and we factorize the histories as
    \begin{align*}
        \sum_{\substack{h \models \varphi_w \\ h=s_0a_0\ldots s_n}} &\prod_{i=0}^{n-1} \overline{P}(s_{i+1}|s_i,a_i) \\
        &=\sum_{\substack{h = s_0 a_0 \ldots s_{n-1}a_{n-1} s_{n} \\ s_{n-1} = s, a_{n-1} = a, s_{n} = s'}}  \prod_{i=0}^n \overline{P}(s_{i+1}| s_i, a_i) \sum_{\substack{h' = s_0' a_0'\ldots s_n'\\ s_0' = s'\\
        h' \models \varphi_{w'}}} \prod_{i=0}^{n-1} \overline{P}(s'_{i+1} | s'_i, a'_i)
    \end{align*}

    We note that the second summation is equal to $\Pr\left(\varphi_w' | \pi, s'\right)$ due to Eq.~\eqref{eq:rewriting_of_success}, which by inductive hypothesis is equal to $p^{|w|_1-1} (1-p)^{n-|w_1|}$. Finally, it can be shown that the first summation is reduced to $p$ when the environment is finite and communicating. Intuitively, observe that it is computing the probability of a history performing the transition $sas'$ before any other $sas''$ with $s'' \neq s'$ when the underlying policy is picking the next action uniformly at random.
\end{proof}

Note that for this proof we had to use a stochastic policy, while the proof from \cite{richens2025general} uses a deterministic and markovian policy which cannot be properly defined in the context of partial observability, as observed in Example~\ref{example:randomness_in_part_obs_is_useful}.

We would like to point out that this result is somewhat counterintuitive: in partially observable environments it can be the case that all states are indistinguishable (for example, if they all provide the same observation) and yet this results says that an optimal agent knows the underlying world. We make some observations that may clarify this situation.

First, note that to satisfy the goals from the proof of Theorem~\ref{teo:non-deterministic-policies-induce-transitions-in-partial-obs} the optimal policy $\pi'$ moves mostly at random, and the only ``reasoning'' comes at the first step, when it must decide whether to go for action $a$ or $b$. This points to the fact that these goals are completely independent from the world, except for the probability $P(s'|s,a)$. The only thing that the agent must understand is the arithmetic to decide which goal is more likely to be achieved when moving at random.

Secondly, even though the world might be fully non-observable, the goals are phrased in terms of the real underlying states, and thus an optimal agent must still know the underlying probabilities to decide which action to take when considering the goals. In particular, this implies that the agent can achieve a goal without knowing it. A way to overcome this ``weirdness'' is by phrasing goals in terms of observation (i.e. goals of the form $\top[V]$ with $V \subseteq \observations \times \actions$). It is easy to see that in that context an agent does not need to know the underlying world to be optimal (consider for example the fully non-observable case). We leave as future work the problem of characterizing the set of transitions that can be recovered from an optimal agent in this case.

Thirdly, observe that to achieve the requested goals moving at random is optimal because there is no discount factor impacting the time taken to reach the objective. Thus, adding some time penalty and defining optimality in a trade-off of achieving goals and achieving them fast might impact the statement of this result, rendering it false.

\subsection{Using goals of width at most 2}

We will define a subset of goals of $\Psi_n$ which is good enough to reconstruct a world model from an optimal policy. Given a goal $\psi \in \Psi$, we define its width as
\begin{align*}
    width(\psi) = width\left(\bigvee_{i\in I} \varphi_i\right) = |I|   
\end{align*}
We denote by $\Psi_{n, k}$ the set of goals of depth $n$ and width $k$.

Note that the results from Theorems~\ref{teo:deterministic_policies_induce_transitions},~\ref{teo:non-deterministic-policies-induce-transitions} and~\ref{teo:non-deterministic-policies-induce-transitions-in-partial-obs} rely on depth-$n$ goals with large width (for example, $width(\rho_{b, r}) = \binom{n}{r}$). We improve these results by obtaining an estimator using only goals of width 2. Moreover, the error of our estimator decays faster than that of Theorem~\ref{teo:deterministic_policies_induce_transitions}.

\begin{theorem}\label{teo:width_2}
    A deterministic goal-conditioned policy optimal on $\Psi_{2n+1, 2}$ allows to recover the transition probabilities with error scaling as $O(\log n/n)$. A deterministic $\delta$-optimal goal-conditioned policy on $\Psi_{6n+1, 2}$ also allows to recover the transition probabilities if $\delta < \frac{1}{2}$ with error growing as $O\left(\frac{\log n}{n}\right)$.
\end{theorem}

\begin{proof}
    Consider goals of the form $\varphi_{b, r} = \varphi_{b, 0^r}$ and $\psi_{b, r} = \varphi_{b, 1^r}$ (using the notation from Lemma~\ref{lemma:prob_of_rho_depends_on_p}). It holds that, for any starting state $s_0$:
    \begin{align*}
        \max_{\pi} \Pr\left( \varphi_{b, r} | \pi,s_0 \right) &= p^r\\
        \max_{\pi} \Pr\left( \psi_{b, r} | \pi,s_0 \right) &= (1-p)^r
    \end{align*}
    where $p = P(s'|s,a)$. We will query the policy for goals of the form $\xi_{r, s} = \varphi_{a, r} \vee \psi_{b, s}$, which satisfy
    \begin{align*}
        \max_{\pi} \Pr\left( \xi_{r, s} | \pi,s_0 \right) = \max\{ p^r, (1-p)^{s} \} 
    \end{align*}

    Note that if the policy goes for action $a$ on its first step when aiming for goal $\xi_{r, s}$, it holds that $p^r \geq (1-p)^s$, and the opposite inequality holds if it goes for action $b$. In particular, we can decide if $p \leq \frac{1}{2}$ by querying for $\xi_{1, 1}$. From now on, we assume that $p \leq \frac{1}{2}$, since otherwise we can rewrite $p \to 1-p$ and apply the same arguments.

    If, when querying $\xi_{1, n}$, the policy goes for action $b$, we know that
    \begin{align*}
        p \leq (1-p)^n \leq e^{-np} \iff pe^{np} \leq 1
    \end{align*}
    Letting $x =np$ this implies that $x e^x \leq n$, and therefore $x \leq W(n)$ where $W$ is the Lambert $W$ function.  It is well known that $W(n) \leq \log n - \log \log n + o(1)$, and thus we get $p \leq \frac{\log n - \log\log n + o(1)}{n}$. Therefore, if we estimate $p$ as 0 we get an error of order $O(\log n / n)$.

    Otherwise, if when querying $\xi_{1, n}$ the policy goes for action $a$, it holds that
    \begin{align*}
        p \geq (1-p)^n \geq 1-np \iff p \geq \frac{1}{n+1}
    \end{align*}
    From now on, we assume this last inequality holds.
    
    Note that for every $1 \leq r ,s  \leq n$ we obtain a constraint of the form
    \begin{align*}
        (1-p)^s \leq p^{r} \iff \frac{\log 1-p}{\log p} \geq \frac{r}{s}
    \end{align*}
    or with the other inequality $\leq$. Since $p\leq \frac{1}{2}$ it holds that $\frac{\log 1-p}{\log p} \leq 1$. By querying with values $r \in \{0,\ldots n\}$ and $s=n$ we obtain an interval estimate of $\alpha = \frac{\log 1-p}{\log p}$ of the form $[l,r]$ with $|r-l| = \frac{1}{n}$. This gives us a guarantee in our approximation of $\alpha$.

    Now we extend this guarantee to an approximation of $p$. Let $f(x) = \frac{\log(1-x)}{\log(x)}$. It can be seen that $f$ is strictly increasing and smooth in the range $\left[\frac{1}{n+1}, \frac{1}{2}\right]$, and thus we let $h = f^{-1}$ which satisfies $h'(y) = \frac{1}{f'(h(x))}$ due to the Inverse Function Theorem. By the Mean Value Theorem, we know that,
    \begin{align*}
        |h(l) - h(r)| \leq |h'(q)| |l-r|
    \end{align*}
    for some $q \in [l, r]$. Since $|l-r| = \frac{1}{n}$, we are only left to upper bound $|h'(q)|$. By the previous remarks $h'(q) = \frac{1}{f'(h(q))}$, and thus we only need to lower bound $f'$ in the range $\left[\frac{1}{n+1}, \frac{1}{2}\right]$.

    By direct computation
    \begin{align*}
        f'(x) = \frac{\frac{-\log(x)}{1-x} - \frac{\log(1-x)}{x}}{\log(x)^2}
    \end{align*}
    Since $\frac{-\log(1-x)}{x} \geq 0$ when $0 \leq x \leq \frac{1}{2}$, we see that
    \begin{align*}
        f'(x) \geq \frac{-\log(x)}{(1-x)\log(x)^2} \geq \frac{1}{-2\log(x)}
    \end{align*}
    where we used that $\frac{1}{1-x} \geq 
    \frac{1}{2}$. Finally, since $x \geq \frac{1}{n+1}$ we see that $\frac{1}{-2\log(x)} \geq \frac{1}{2\log (n+1)}$. Thus, we conclude that
    \begin{align*}
        |h(l)-h(r)| \leq \frac{2\log (n+1)}{n} 
    \end{align*}

    To prove the second statement we apply a similar strategy, though the proof is more cumbersome. When querying $\pi$ with the goal $\xi_{r,s}$ we will deduce an inequality of the form
    \begin{align}\label{eq:delta_cond_for_proof}
        (1-\delta)p^r \leq(1-p)^s
    \end{align}
    or an identical one but with $p$ turned into $1-p$. As before, we start by bounding $p$ from above: if
    \begin{align*}
        (1-\delta)p \leq 1-p 
    \end{align*}
    then, since $\delta \leq \frac{1}{2}$, we can conclude that $p \leq \frac{2}{3}$. Otherwise, we map $p \to 1-p$ and apply the same arguments. This is the only part in which we use our hypothesis over $\delta$.

    As before, we will bound $p$ away from $0$. If
    \begin{align*}
        (1-\delta) p \leq (1-p)^n \leq e^{-pn}
    \end{align*}
    then $(1-\delta)p e^{pn} \leq 1$ and $p \leq \frac{W\left( \frac{n}{(1-\delta)} \right)}{n} = O\left(\frac{\log \left( \frac{n}{1-\delta} \right)}{n}\right)$. Thus, we estimate $p$ as 0.

    Otherwise, if
    \begin{align*}
        p\geq (1-\delta)(1-p)^n \geq (1-\delta) (1-pn)
    \end{align*}
    we conclude that $p \geq \frac{1-\delta}{1+(1-\delta)n}$.

    We will estimate $\alpha = \frac{\log 1-p}{\log p}$. From $p \leq \frac{2}{3}$ we know that $\alpha \leq 3$. Let's write Eq.~\eqref{eq:delta_cond_for_proof} in terms of $\alpha$:
    \begin{align*}
        (1-\delta)p^r \leq (1-p)^s &\iff \log (1-\delta) + r\log p \leq s \log 1-p\\
        &\iff \frac{\log (1-\delta)}{s \log p} \geq \alpha - \frac{r}{s}
    \end{align*}
    Similarly, from the other kind of inequalities we obtain
    \begin{align*}
        (1-\delta) (1-p)^s \leq p^r \iff \frac{\log (1-\delta)}{s \log p} \geq \frac{r}{s} - \alpha 
    \end{align*}

    To make these constraints strong, we note that $\frac{\log (1-\delta)}{s \log p} \leq \frac{|\log (1-\delta)|}{n}$ if $p \leq \frac{2}{3}$ and $s \geq n$. If we query the policy with values of the form $r \in \{0\ldots 3n\}$ and $s = n$ we will partition the interval $[0,3]$ in slices of size $\frac{1}{n}$, and the condition on $s$ we recently stated will hold.

    As estimator, we will consider the first fraction $r_0 / n$ such that $(1-\delta)p^{r_0} \leq (1-p)^n$, if there is such fraction (note that for $r=1$ we know this inequality does not hold). If it exists, then
    \begin{align*}
        \frac{r_0-1}{n} - \frac{|\log (1-\delta)|}{n} \leq \alpha \leq \frac{r_0}{n} + \frac{|\log(1-\delta)|}{n} 
    \end{align*}
    and therefore the estimator for $\alpha$ has error bounded by $\frac{1}{n} + \frac{|\log(1-\delta)|}{n}$. If there is no such fraction we take as estimator $3$, and we can still bound its error using the fact that $\alpha \leq 3$.

    Finally, we apply the same reasoning as before to convert this confidence interval for $\alpha$ into a confidence interval for $p$, but observing that now we work in the range $\left[ \frac{1-\delta}{1 + n(1-\delta)} , \frac{2}{3} \right]$. We can bound the derivative as
    \begin{align*}
        |h'(q)| \leq 3\log\left( \frac{1 + n(1-\delta)}{1-\delta} \right)
    \end{align*}
    and the difference between $p$ and its estimator $\hat{p} = f^{-1}(r_0 / n)$ is
    \begin{align*}
        |p-\hat{p}| \leq \frac{3\log \left( \frac{1+n(1-\delta)}{(1-\delta)} \right) \left( 1 + \frac{|\log(1-\delta)|}{n}\right)}{n}
    \end{align*}
    For $\delta < \frac{1}{2}$, this error grows as $O\left( \frac{\log n}{n} \right)$.
\end{proof}

Note that the constraint $\delta < \frac{1}{2}$ was imposed only to simplify the proof. The technique probably works for any $\delta$, even though the final expression might become worse.

This is a clear improvement over Theorem~\ref{teo:deterministic_policies_induce_transitions}. Not only the convergence of the error is faster, but also the goals on which the agent must be optimal are simpler. In particular, $|\Psi_n| = \Theta\left( 2^{2^{n|\states||\actions|}}\right)$, while $|\Psi_{n, 2}| = \Theta\left( 2^{2n|\states||\actions|}\right)$. Moreover, the set of goals that we employ in the proof of Theorem~\ref{teo:width_2} are actually small (their size is linear in $n$) while those from Theorem~\ref{teo:deterministic_policies_induce_transitions} have exponential size.
 
\section{Conclusion}\label{sec:conclu}

In this work we presented two extensions of the main result from~\cite{richens2025general}, which we recalled in Theorem~\ref{teo:deterministic_policies_induce_transitions}. The proof of this theorem relied on the hypothesis that the agent is deterministic and the world is fully observable.

Our first extension, Theorem~\ref{teo:non-deterministic-policies-induce-transitions}, shows that the result still holds for stochastic policies. More precisely, we showed that from any randomized capable enough policy it is possible to obtain an approximation of the underlying world in which the agent operates. When comparing this result to Theorem~\ref{teo:deterministic_policies_induce_transitions} we see that the quality of the estimator is worse, in particular when looking at the dependence on $\delta$. This is intuitive: a stochastic $\delta$-optimal policy can avoid learning precisely which strategies are optimal by acting on a mixture of them. This allows the policy to ``trade-off'' learning the environment with acting at random.

Our second extension, Theorem~\ref{teo:non-deterministic-policies-induce-transitions-in-partial-obs}, shows that the result still holds for partially observable worlds. This result seems counter-intuitive, since there are worlds which are so unobservable that the agent might never be able to tell where it is. Nonetheless, if we keep the exact framework of~\cite{richens2025general} we observed that the core goals used in the proof of Theorem~\ref{teo:deterministic_policies_induce_transitions} do not rely on moving intelligently around the world, but rather on being able to tell which event is more likely when moving at random. Thus, thanks to this observation we were able to extend the result to the partially observable case.

We note, however, that there is more than one way to extend the framework to the partially observable case. For example, in our proposal goals are still phrased in terms of the real states of the cMDP, even though the agent is only able to look at observations. It would be interesting to understand what happens when an agent is only optimal for attaining ``observation-phrased'' goals. Note that in that case, optimal agents might not need to know the true transition probabilities: for example, if all states output the same observation, then the optimal policies do not depend on the values of $P$, and thus we cannot recover them from the policies. A concrete project is to characterize the set of environments or states which can be approximated given a good enough policy in this context.

Finally, we also improved the notion of ``general agent'' employed in Theorem~\ref{teo:deterministic_policies_induce_transitions} by showing that less general agents already have enough information to reconstruct the environment. In particular, we only require the agent to be optimal on goals with width 2 (i.e. goals which only have 2 subgoals). This subset is not only smaller, but the goals employed in the proof itself are exponentially smaller than the ones used in the original statement.

\bibliographystyle{plain}
\bibliography{bilbio}

@article{kaelbling1998planning,
  title={Planning and acting in partially observable stochastic domains},
  author={Kaelbling, Leslie Pack and Littman, Michael L and Cassandra, Anthony R},
  journal={Artificial intelligence},
  volume={101},
  number={1-2},
  pages={99--134},
  year={1998},
  publisher={Elsevier}
}

@inproceedings{richens2025general,
  title={General agents need world models},
  author={Richens, Jonathan and Everitt, Tom and Abel, David},
  booktitle={Forty-second International Conference on Machine Learning},
  year={2025}
}

@book{johnson1983mental,
  title={Mental models: Towards a cognitive science of language, inference, and consciousness},
  author={Johnson-Laird, Philip Nicholas},
  year={1983},
  publisher={Harvard University Press}
}

@article{ha2018world,
  title={World models},
  author={Ha, David and Schmidhuber, J{\"u}rgen},
  journal={arXiv preprint arXiv:1803.10122},
  volume={2},
  number={3},
  year={2018}
}

@article{hafner2023mastering,
  title={Mastering diverse domains through world models},
  author={Hafner, Danijar and Pasukonis, Jurgis and Ba, Jimmy and Lillicrap, Timothy},
  journal={arXiv preprint arXiv:2301.04104},
  year={2023}
}

@article{wang2023voyager,
  title={Voyager: An open-ended embodied agent with large language models},
  author={Wang, Guanzhi and Xie, Yuqi and Jiang, Yunfan and Mandlekar, Ajay and Xiao, Chaowei and Zhu, Yuke and Fan, Linxi and Anandkumar, Anima},
  journal={arXiv preprint arXiv:2305.16291},
  year={2023}
}

@article{lecun2022path,
  title={A path towards autonomous machine intelligence version 0.9. 2, 2022-06-27},
  author={LeCun, Yann},
  journal={Open Review},
  volume={62},
  number={1},
  pages={1--62},
  year={2022}
}

@article{reed2022generalist,
  title={A generalist agent},
  author={Reed, Scott and Zolna, Konrad and Parisotto, Emilio and Colmenarejo, Sergio Gomez and Novikov, Alexander and Barth-Maron, Gabriel and Gimenez, Mai and Sulsky, Yury and Kay, Jackie and Springenberg, Jost Tobias and others},
  journal={arXiv preprint arXiv:2205.06175},
  year={2022}
}

@inproceedings{zitkovich2023rt,
  title={Rt-2: Vision-language-action models transfer web knowledge to robotic control},
  author={Zitkovich, Brianna and Yu, Tianhe and Xu, Sichun and Xu, Peng and Xiao, Ted and Xia, Fei and Wu, Jialin and Wohlhart, Paul and Welker, Stefan and Wahid, Ayzaan and others},
  booktitle={Conference on Robot Learning},
  pages={2165--2183},
  year={2023},
  organization={PMLR}
}

@article{li2022emergent,
  title={Emergent world representations: Exploring a sequence model trained on a synthetic task},
  author={Li, Kenneth and Hopkins, Aspen K and Bau, David and Vi{\'e}gas, Fernanda and Pfister, Hanspeter and Wattenberg, Martin},
  journal={arXiv preprint arXiv:2210.13382},
  year={2022}
}

@article{sutton1999policy,
  title={Policy gradient methods for reinforcement learning with function approximation},
  author={Sutton, Richard S and McAllester, David and Singh, Satinder and Mansour, Yishay},
  journal={Advances in neural information processing systems},
  volume={12},
  year={1999}
}

@article{rabin1980probabilistic,
  title={Probabilistic algorithm for testing primality},
  author={Rabin, Michael O},
  journal={Journal of number theory},
  volume={12},
  number={1},
  pages={128--138},
  year={1980},
  publisher={Elsevier}
}

@article{agrawal2004primes,
  title={PRIMES is in P},
  author={Agrawal, Manindra and Kayal, Neeraj and Saxena, Nitin},
  journal={Annals of mathematics},
  pages={781--793},
  year={2004},
  publisher={JSTOR}
}

@article{schwartz1980fast,
  title={Fast probabilistic algorithms for verification of polynomial identities},
  author={Schwartz, Jacob T},
  journal={Journal of the ACM (JACM)},
  volume={27},
  number={4},
  pages={701--717},
  year={1980},
  publisher={ACM New York, NY, USA}
}

\appendix

\end{document}